\documentclass[conference]{IEEEtran}  

\usepackage{amsmath}
\usepackage{graphicx}
\usepackage{verbatim}
\usepackage{latexsym}
\usepackage{setspace}
\usepackage{amssymb}
\usepackage{mathtools}
\usepackage{bbm}
\usepackage{caption}
\usepackage{subcaption}
\usepackage{bm}
\usepackage{cite}
\usepackage{tikz}
\usetikzlibrary{arrows,shapes}
\usepackage{pgfplots}
\usepackage{amsthm}
\usepackage{algorithm}
\usepackage{algpseudocode}

\usepackage[colorinlistoftodos]{todonotes}

\newtheorem{definition}{Definition}

\newtheorem{proposition}{Proposition}
\newtheorem{problem}{Problem}
\newtheorem{lemma}{Lemma}

\newtheorem{remark}{Remark}
\newtheorem{assumption}{Assumption}
\newtheorem*{assumption*}{Assumption}

\title{\LARGE \bf
A Self-Guided Approach for Navigation in a Minimalistic Foraging Robotic Swarm
}

\author{\IEEEauthorblockN{Steven Adams}\IEEEauthorblockA{Delft Center for Systems and Control,\\
Delft University of Technology,\\
Delft, 2628 CD, The Netherlands.\\
\textit{s.j.l.adams@tudelft.nl}
        }\and
\IEEEauthorblockN{ Daniel Jarne Ornia} \IEEEauthorblockA{Delft Center for Systems and Control,\\
Delft University of Technology,\\
Delft, 2628 CD, The Netherlands.\\
\textit{d.jarneornia@tudelft.nl}
        }
        \and
\IEEEauthorblockN{Manuel Mazo, Jr.}

\IEEEauthorblockA{Delft Center for Systems and Control,\\
Delft University of Technology,\\
Delft, 2628 CD, The Netherlands.\\
\textit{m.mazo@tudelft.nl}
        }%
}

\begin{document}

\maketitle

\begin{abstract}
We present a biologically inspired design for swarm foraging based on ant's pheromone deployment, where the swarm is assumed to have very restricted capabilities. The robots do not require global or relative position measurements and the swarm is fully decentralized and needs no infrastructure in place. Additionally, the system only requires one-hop communication over the robot network, we do not make any assumptions about the connectivity of the communication graph and the transmission of information and computation is scalable versus the number of agents. This is done by letting the agents in the swarm act as foragers or as guiding agents (beacons). We present experimental results computed for a swarm of Elisa-3 robots on a simulator, and show how the swarm self-organizes to solve a foraging problem over an unknown environment, converging to trajectories around the shortest path. At last, we discuss the limitations of such a system and propose how the foraging efficiency can be increased.
\end{abstract}

\section{INTRODUCTION}
In the past thirty years the use of multi-agent techniques to solve robotic tasks has exploded. The advancement of processing power, sensor accuracy and battery sizes has enabled the realisation of coordinated multi-robot problems, where a number of agents organize to solve a specific task. When designing increasingly larger systems, a huge part of methods draw inspiration from biological systems (ants, bees...) where large amounts of agents interact (directly or indirectly) to produce emerging behaviour that results in an optimized solution to a physical problem. Biologically inspired multi-robot systems are usually referred to as robotic swarms (See e.g. \cite{beni1993swarm,dorigo2007swarm,blum2008swarm,kennedy2006swarm}). 

These biological methods have enabled plenty of theoretical developments, but the applicability of them is still sparse partly due to problem complexities hard to satisfy with minimalistic robots. We focus in this work on the well known foraging problem, defined as the dual problem of exploration/exploitation, where a number of agents start at a given point in space and must find a target in an unknown (possibly time varying) environment, while converging to cycle trajectories that enable them to exploit the found resources as efficiently as possible. Foraging has been extensively studied but it is still interesting when designing very large robotic systems since it combines exploration and on-line path planning, and the duality of exploration vs. exploitation is nowadays extremely relevant in Reinforcement Learning and other AI related fields \cite{thrun1992efficient,ishii2002control,nair2018overcoming}. In particular, this problem has been addressed with robotic ant-inspired swarms that use indirect communication through some ``pheromone" (either virtual or chemical) \cite{fujisawa2008communication,russell1997heat,johansson2009navigating}. Some early work was done in \cite{drogoul1993some} showing how robots can use pheromone based information to explore and collect. After that, authors in e.g. \cite{sugawara2004foraging,fujiswarms,campo2010artificial,
alers2014insect,10.1007/978-3-030-00533-7_11,10.1007/978-3-642-15461-4_8,4223153,4209807,DUCATELLE200925,svennebring2004building} have explored several practical methods to implement pheromone-based robotic navigation. However, very often complexities explode when designing very large multi-robot swarms, be it in terms of computation, data transmission or dynamic coupling. Additionally, these systems sometimes include implicit assumptions that in practice prevent them of being applied to large scenarios or real situations. These may be related to sensor range, memory storage in the agents, computational capacity or reliability of communications, among others.

\subsection{Related Work}
There have been multiple proposals for de-centralised multi robot foraging systems. Authors in \cite{payton2001pheromone} propose a multi-robot system that uses IR to communicate a pheromone based counter signal between robots in a network, cascading all agents' information through the network, such that agents find an unknown target in space and find trajectories back and forth. The authors in \cite{ducatelle2011communication} use an ant-inspired swarm to solve a foraging problem on a 2D space by assuming a connected line-of-sight communication network, and having agents flood this network with their estimation of their relative position and angle at every time step. 

In \cite{hrolenok2010collaborative,russell2015swarm}, authors use a combination of agents and beacon devices to guide navigation and store pheromone values. Authors treat pheromones as utility estimates for environmental states, and agents rely on line-of-sight communication and relative distances to the beacons to generate foraging trajectories. 
In \cite{nouyan2008path} the authors design a system where robots communicate with each other via LED signalling to indicate trails or vector fields pointing towards a given target, and provide empirical results on a wide number of scenarios. In \cite{reina2017ark,talamali2020sophisticated} the authors use a virtual-reality approach to implement the pheromone field,  allowing the robots to have access to this virtual pheromone from a central controller, enabling effective foraging.

It is worth noting that \cite{nouyan2008path, ducatelle2011communication, payton2001pheromone} assume agents in the swarm communicate directly with other agents, while \cite{hrolenok2010collaborative,russell2015swarm} de-couples this and proposes an environment-based interactions where agents only write and read data into locally reachable beacons. All represent collaborative methods, and assume knowledge of relative positions between agents. In \cite{lemmens2007bee,lemmens2009stigmergic}, authors propose bee-inspired path integration algorithms on a computational set-up, where agents use landmarks to store pheromone-based information when a change in direction is required.

\subsection{Main Proposal}
All the reviewed work requires robots to either have some form of position measurement (global or relative), some form of infrastructure or centralised knowledge entity, or both. These are specially critical when navigating in environments where any access to global or relative positioning is not possible, and where the infrastructure is limited (\emph{e.g.} space exploration). Additionally, some present strong requirements on communication of information in the swarm, either requiring connected line-of-sight communication or by cascading tables of agent data through the network. Our goal is here to design a minimalistic swarm system capable of solving the foraging problem by using a form of pheromone-inspired communication, without any such assumptions, and the following restrictions:
\begin{itemize}
\item Minimal assumptions on the robot capabilities. All agents are supposed to have equal characteristics (homogeneous system), do not have knowledge of relative (or global) positions with respect to other robots, and only need an orientation measure (a compass).
\item The system relies on one-hop communication only with limited range, and does not require direction or distance information on signals, nor line-of-sight connectivity.
\item Is fully distributed and needs no infrastructure in place.
\item Does so with robustness versus communication events or single agent failures.
\end{itemize}
\subsection{Preliminaries}
We use calligraphic letters for sets ($\mathcal{A}$), regular letters for scalars ($a\in\mathbb{R}$) and bold letters for vectors ($\mathbf{a}\in\mathbb{R}^{n}$). We consider discrete time dynamics $k\in\mathbb{N}$, and we define an inter-sampling time $\tau\in\mathbb{R}_+$ such that we keep a ``total" time measure  $t=\tau k$. With vectors we use $\|\mathbf{v}\|$ as the euclidean norm, and $\langle\mathbf{v}\rangle:=\frac{\mathbf{v}}{\|\mathbf{v}\|}$. We use the diagonal operator $D(\cdot)=\operatorname{diag}(\cdot)$.

\section{Problem Description}\label{sec:prob}
Take a swarm of $N$ agents $\mathcal{A}=\{1,2,...,N\}$ navigating in a bounded domain $\mathcal{D}\subset \mathbb{R}^2$, where $\mathcal{D}$ is compact (possibly non-convex). We define $\mathbf{x}_a(k)\in\mathcal{D}$ as the position of agent $a$ at time $t$, and velocity $\mathbf{v}_a(k) =v_0 (\cos (\alpha_a(k))\,\,\,\sin(\alpha_a(k)))^T$ with $\alpha_a(k)\in [-\pi,\pi )$ as its orientation. We define the dynamics of the system to be in discrete time, such that the positions of the agents evolve as
\begin{equation}
\mathbf{x}_a(k+1) = \mathbf{x}_a(k)+\mathbf{v}_a(k) \tau.
\end{equation}
Consider the case where the swarm is trying to solve a \emph{foraging problem}. 
\begin{definition}[Foraging Problem]
A foraging problem on an unknown domain $\mathcal{D}$ is the joint problem of finding a target region $\mathcal{T}\subset\mathcal{D}$ when starting on a different region $\mathcal{S}\subset\mathcal{D}$, $\mathcal{S}\cap\mathcal{T}=\emptyset$, and eventually following (semi) optimal trajectories between $\mathcal{S}$ and $\mathcal{D}$.
\end{definition}
The main goal when solving a foraging problem is to complete trajectories between $\mathcal{S}$ and $\mathcal{D}$ as fast as possible (through the shortest path), back and forth. We consider the foraging problem to be solved if, eventually, all agents in the swarm are able to follow trajectories that are close to the set of optimal trajectories. To design such a swarm, we make the following assumptions on the swarm agents' capabilities.
\begin{assumption}[Swarm Constrains]\label{as:1}\
\begin{enumerate}
\item Agents have a small memory, enough to store scalars and vectors in $\mathbb{R}^2$, and enough computational power to perform sums and products of vectors.
\item Agents have the ability to send and receive basic signals (up to 6 scalar values), within a maximum range $\delta$.
\item Agents have some form of collision avoidance mechanism,  acting independently of the design dynamics.
\item Agents have sensing ability to detect $\mathcal{S}$, $\mathcal{T}$.
\item Agents have some measure of angular orientation (\emph{e.g.} a compass).
\item Agents are able to remain static.
\end{enumerate}
\end{assumption}
Observe that we do not assume the ability to measure directionality in the signals, nor any form of self-localisation capacity. Additionally, the agents do not have access to any form of global information about $\mathcal{D}$, do not have unique identifiers and do not require line-of-sight interactions. The swarm does not require either any form of infrastructure in place. At last, the swarm relies on one-hop communication only. That is, the communication happens on a agent-to-agent basis, and agents do not cascade information through the communication network. The challenge to be solved in this work is then the following.
\begin{problem}
Design a swarm of $N$ agents that solves a foraging problem over an unknown domain $\mathcal{D}$, while satisfying the set of Assumptions \ref{as:1}, and does so with guarantees.
\end{problem}

\section{Proposal: Self Guided Swarm}\label{sec:swarmdesign}
We now present our design for a self-guided swarm that solves the foraging problem presented in Section \ref{sec:prob}. Our design is based on the idea of allowing agents in the swarm to behave as ``beacons" (agents in charge of guiding others) or ``foragers" (agents in charge of travelling from $\mathcal{S}$ to $\mathcal{T}$). Beacon agents store weight values $\omega^s_b(k)\in\mathbb{R}_+$ and guiding velocity vectors $\mathbf{u}_b^s(k)\in\mathbb{R}^2$, which they broadcast to the foragers in the swarm to generate foraging trajectories. We first describe the different modes the agents can operate in, the switching rules between modes, and then the dynamics in every mode.
\subsection{States and transitions}
Let us then split the swarm into three groups of agents. We will use the state variable $s_a(k)\in\{B,F_1,F_2\}$ to indicate:
\begin{enumerate}
\item $s_a(k) ={B}\Rightarrow$ $a$ is a beacon.
\item $s_a(k) = {F_1}\Rightarrow$ $a$ is a forager looking for $\mathcal{T}$.
\item $s_a(k) = {F_2}\Rightarrow$ $a$ is a forager looking for $\mathcal{S}$.
\end{enumerate}
We use $s=F_1,F_2$ as the different foraging states. Then, we can group the agents in time-dependent sub-sets: the beacon set $\mathcal{B}(k):\{a\in\mathcal{A}:s_a(k)=B\}$ and the forager sets $\mathcal{A}^s(k):=\{a\in\mathcal{A}:s_a(k)=s\}$.  

At $t=0$ all agents are initialised at $\mathcal{S}$. One agent is chosen to be the initial beacon, and all others are initialised as foragers looking for $\mathcal{T}$:
\begin{equation}
\mathbf{x}_a(0) \in \mathcal{S}\,\forall a\in\mathcal{A},\,\,\lvert\mathcal{B}(0)\rvert = 1,\,\lvert\mathcal{A}^{F_1}(0)\rvert = N-1.
\end{equation}
This initial beacon can be chosen at random, or based on some order of deployment of the swarm. Let us now define the regions of influence of every agent as $\mathcal{D}_a(k):=\{\mathbf{x}\in\mathcal{D}:\|\mathbf{x}-\mathbf{x}_a(k)\|_2\leq \delta\}$, for some maximum instrument range $\delta \in\mathbb{R}_+$. As time evolves, the agents switch between states following the logic rules $\forall a\in\mathcal{A}$:
\begin{equation}\label{eq:statetrans}
s_a(k+1) = \left\{\begin{array}{l}B\quad \text{if}\,\,\nexists\,\, b\in\mathcal{B}(k):\mathbf{x}_b(k)\in \mathcal{D}_a(k),\\
F_1\quad \text{if}\,\,s_a(k)=F_2\wedge \mathbf{x}_a(k)\in\mathcal{S},\\
F_2\quad \text{if}\,\,s_a(k)=F_1\wedge \mathbf{x}_a(k)\in\mathcal{T},\\
s_a(k)\quad \text{else},
\end{array}\right. .
\end{equation}
The switching rule in \eqref{eq:statetrans} is interpreted in the following way. If a forager moves to a point in the domain where there is no other beacons around, it becomes a beacon. If a forager is looking for the set $\mathcal{T}$ (mode 1) and finds it, it switches to finding the starting set $\mathcal{S}$ (mode 2), and the other way around. For now we do not consider transitions from beacon to forager.
\subsection{Dynamics}
We assume that beacons remain static while in beacon state:
\begin{equation}\label{eq:dynbeac}
\mathbf{v}_b(k) = \textbf{0},\,\,\mathbf{x}_b(k)=\mathbf{x}_b(k_b)\,\,\,\forall k\geq k_b, 
\end{equation}
where $k_b$ is the time step when agent $b$ switched to beacon state. Beacon agents store weight values $\omega^s_b(k)\in\mathbb{R}_+$ and guiding velocity vectors $\mathbf{u}_b^s(k)\in\mathbb{R}^2$, initialised at zero for all agents in the swarm. At every time-step, beacon agents broadcast their stored values $\omega^s_b(k),\mathbf{u}_b^s(k)$ with a signal of radius $\delta$. Let us define the set of neighbouring beacons to a forager in state $s$, $f\in\mathcal{A}^s(k)$ as $\mathcal{B}_f(k):=\{b\in\mathcal{B}(k):\mathbf{x}_b(k)\in\mathcal{D}_f(k)\}$, and the set of neighbouring foragers to a beacon $b\in\mathcal{B}(k)$ as $\mathcal{F}_b^s(k):=\{f\in\mathcal{A}^s(k):\mathbf{x}_f(k)\in\mathcal{D}_b(k)\}$. 

To account for possible disturbances in sensor measurements, we consider the disturbed weight and velocity values for a random variable $\mu_k\sim\mathcal{N}(1,\sigma^2)$,
\begin{equation}\begin{aligned}
\tilde{\omega}^s_b(k,\sigma):=&\lvert\mu_k\rvert\omega^s_b(k),\\ \tilde{\mathbf{v}}_f(k,\sigma) :=&D(\mu_k,\mu_k)\mathbf{v}_f(k).
\end{aligned}
\end{equation}

At every time-step each forager receives a set of signals from neighbouring beacons, and computes a reward function $\Delta^s_f(k)\in\mathbb{R}_+$:
\begin{equation}\label{eq:rew}
\Delta^s_f(k)=\gamma^s_f(k) + \lambda\max_{b\in\mathcal{B}_f(k)}\tilde{\omega}_b^s(k,\sigma),
\end{equation}
where $\lambda\in [0,1]$ is a discount factor, and for some $r\in\mathbb{R}_+$,
\begin{equation}
\gamma^s_f(k) = \left\{\begin{array}{l}r \quad \text{if}\,\,s_f(k)=F_1 \wedge \mathbf{x}_f(k)\in\mathcal{S},\\
r \quad \text{if}\,\, s_f(k)=F_2 \wedge \mathbf{x}_f(k)\in\mathcal{T},\\
0\quad\text{else}.
\end{array}\right.
\end{equation}
The reward function in \eqref{eq:rew} works as follows: Foragers listen for weight signals from neighbouring beacons and broadcast back the maximum discounted weight, plus a constant reward if they are in the regions $\mathcal{S}$ or $\mathcal{T}$ depending on their state. Observe that \eqref{eq:rew} depends on $s$, indicating that foragers listen and reinforce only the weights corresponding to their internal state value. The beacons update their weight values using a (possibly different) discount factor $\rho\in [0,1]$ as
\begin{equation}\label{eq:wupdate}
\omega^s_b(k+1) = (1-\rho) \omega^s_b(k)+\rho\frac{\sum_{f\in \mathcal{F}_b^s(k)}\Delta_f^s(k)}{\lvert\mathcal{F}_b^s(k)\rvert}.
\end{equation}
The iteration in \eqref{eq:wupdate} is only applied if there are indeed neighbouring foragers around a beacon, so $\lvert\mathcal{F}_b^s(k)\rvert\geq 1$. Otherwise, $\omega^s_b(k+1) =\omega^s_b(k)$.

The update rule of $\mathbf{u}_b^s(k)$ is similarly:
\begin{equation}\label{eq:nupdate}
\mathbf{u}^s_b(k+1) = (1-\rho) \mathbf{u}^s_b(k)-\rho\frac{\sum_{f\in \mathcal{F}_b^s(k)}\tilde{\mathbf{v}}_f(k,\sigma)}{\lvert\mathcal{F}_b^s(k)\rvert}.
\end{equation}
and again, $\lvert\mathcal{F}_b^s(k)\rvert=0\Rightarrow \mathbf{u}^s_b(k+1)=\mathbf{u}^s_b(k)$. At the same time that beacons update their stored weight values based on the foragers around them, they update as well the guiding velocity vectors by adding the corresponding velocities of the foragers around them (with an opposite sign). The logic behind this has to do with the reward function in \eqref{eq:rew}. Foragers looking for $\mathcal{T}$ update weights and guiding velocities associated with state $F_1$, but to indicate that they are in fact moving out of $\mathcal{S}$, we want to update the guiding velocities based on the opposite direction that they are following.

Until now we have defined the dynamics of the beacon agents: position and velocities in \eqref{eq:dynbeac} and update rules for $\omega^s_b(k),\,\mathbf{u}^s_b(k)$ in \eqref{eq:wupdate} and \eqref{eq:nupdate}. We have yet to define the dynamics of the foragers. At every step, the foragers listen for guiding velocity and weight signals from beacons around them. With this information, they compute the guiding vector:
\begin{equation}
\hat{\mathbf{v}}^s_f(k) :=v_0 \left\langle \sum_{b\in\mathcal{B}_f(k)}\tilde{\omega}_b^{\overline{s}}(k,\sigma)\mathbf{u}^{\overline{s}}_b(k)\right\rangle,
\end{equation}
where $\overline{s}$ is the opposite state to $s$, $s=F_1\iff \overline{s}=F_2$. At every time-step foragers choose stochastically, for a design exploration rate $\varepsilon\in (0,1)$, if they follow the guiding vector $\hat{\mathbf{v}}^s_f(k)$ or they introduce some random noise to their movement. Let $\alpha_u$ be a random variable taking values $(-\pi,\pi]$, following some probability density function $p(\alpha_u)$, and let $\tilde{\alpha}_a (k):=\alpha_a(k)+\alpha_u$. Then, $\forall f\in \mathcal{A}^s(k)$:
\begin{equation}\label{eq:prv}
\begin{array}{l}\Pr \{\mathbf{v}_f(k+1)= v_0\left(\begin{array}{c}
\cos \left(\tilde{\alpha}_a (k)\right)\\
\sin\left(\tilde{\alpha}_a (k)\right)
\end{array}\right)\}=\varepsilon,\\
\Pr \{\mathbf{v}_f(k+1)= \hat{\mathbf{v}}^s_f(k)\}=1-\varepsilon.
\end{array}
\end{equation}
Additionally, we add a fixed condition for an agent to turn around when switching between foraging states. That is, 
\begin{equation}\label{eq:v2}
\mathbf{v}_f(k+1) = -\mathbf{v}_f(k)\quad\text{if} \,\,s_f(k+1) \neq s_f(k).
\end{equation}
With \eqref{eq:prv} and \eqref{eq:v2} the dynamics of the foragers are defined too. We have experimentally verified that, over a diverse variety of set-ups, the weights $\omega_b^s(k)$ converge on average to a fixed point forming a gradient outwards from $\mathcal{S}$ and $\mathcal{T}$, therefore guiding the swarm to and from the goal regions.
\section{Results and Guarantees}
To show that the swarm finds the target region $\mathcal{T}$ and eventually converges to close to optimal trajectories between $\mathcal{S}$ and $\mathcal{T}$, let us first state the assumptions for the following theoretical results.
\begin{assumption}\label{as:2}\
\begin{enumerate}
\item Any region $\mathcal{D}_k\subseteq \mathcal{D}$ is a compact disc of radius $\delta$.
\item  $\tau<\frac{\delta}{2}$, and can be chosen small enough in comparison to the diameter of the domain $\mathcal{D}$.
\item The regions $\mathcal{S}$ and $\mathcal{T}$ are compact discs of radius, at least, $2\tau$.
\end{enumerate}
\end{assumption}
\subsection{Domain Exploration}
\begin{remark}\label{rem:1}
It holds by construction that $\exists b: \mathbf{x}_a(k) \in\mathcal{D}_b(k)\,\,\forall t,\,\forall a\in\mathcal{A}^s(k)$.
\end{remark}
From the transition rule in \eqref{eq:statetrans} it follows that, whenever $\mathbf{x}_a(k)\notin \mathcal{D}_b(k)$ for any beacon $b\in\mathcal{B}(k)$, it becomes a beacon, therefore covering a new sub-region of the space. To obtain the first results regarding the exploration of the domain, we can take inspirations from results in RRT exploration \cite{kuffner2000rrt}, where the agents perform a similar fixed length step exploration over a continuous domain. Let $p_a(\mathbf{x},k)$ be the agent probability density of point $\mathbf{x}$ at time $k$. We obtain the following results.
\begin{proposition}\label{prop:explore}
Let $\mathcal{D}$ be convex. Let some $a\in\mathcal{A}$ have $\mathbf{x}_a(k_0) = \mathbf{x}_0$. Then, for any convex region $\mathcal{D}_n\subset \mathcal{D}$ of non-zero volume, there exists $\kappa_n \in \mathbb{R}_+$ and time $k_n\in\mathbb{N}$ such that
\begin{equation}
\,\Pr[\mathbf{x}_a(k_n+k_0)\in\mathcal{D}_n\,\mid \,\mathbf{x}_a(k_0)]\geq \varepsilon^{k_n} \kappa_n.
\end{equation}
\end{proposition}
\begin{proof}[Sketch of proof]
At every time-step $t$, with probability $\varepsilon$, it holds $\alpha_a(k+1) = \alpha_a(k) + \alpha_u,$ and observe that $\alpha_a(k+1)\in (-\pi,\pi]$. Let agent $a$ be at point $\mathbf{x}_0$ at time $k_0$, and take the event where for every time-step after $k_0$ the agent always chooses the random velocity. For $k=k_0+1$, the set of points $\mathcal{X}_a(1)\subseteq\mathcal{D}$ satisfying $p_a(k_0+1,\mathbf{x})>0$ is $\mathcal{X}_a(k_0+1) = \{x\in\mathcal{D}:\|x-\mathbf{x}_0\|_2 = \tau\}.$ One can verify that $\mathcal{X}_a(k_0+1)$ forms a circle in $\mathbb{R}^2$ around $\mathbf{x}_0$. 
Now for $kl=k_0+2$, the set $\mathcal{X}_a(k_0+2)$ is
\begin{equation}\label{eq:ball}
\mathcal{X}_a(k_0+2) = \{\mathbf{x}\in\mathcal{D}:\|\mathbf{x}-\mathbf{x}_1\|_2 = \tau ,\, \mathbf{x}_1\in\mathcal{X}_a(k_0+1) \}.
\end{equation}
In this case, the set in \eqref{eq:ball} forms a ball of radius $2\tau$ around $\mathbf{x}_0$.
At $k=k_0+2$, it holds from \eqref{eq:ball} that $p_a(\mathbf{x},k_0+2)>0\,\forall \,\mathbf{x}\in\mathcal{X}_a(2)$. Then, for any subset $\mathcal{D}_2\subseteq\mathcal{X}_a(2)$, it holds that $\Pr [\mathbf{x}_a(2)\in \mathcal{D}_2\,\mid\,\mathbf{x}_a(k_0)] \geq\varepsilon^2 \kappa_2,$ where $\kappa_2\in\mathbb{R}_+$ is a function of the set $\mathcal{D}_2$ and the probability density $p_a(\mathbf{x},k_0+2)$. The sets $\mathcal{X}_a(k)$ are balls centred at $\mathbf{x}_0$ and radius $k\,\tau$, and $p_a(\mathbf{x},k)>0\,\forall \,x\in\mathcal{X}_a(k)$. Let at last $\mathcal{D}_n$ be any convex subset of $\mathcal{D}$ with non-zero volume, and $k_n=\min\{k:\mathcal{D}_n \subset\mathcal{X}_a(k)\} $. Then, $\Pr [\mathbf{x}_a(k_0+k_n)\in \mathcal{D}_n\,\mid\,\mathbf{x}_a(k_0)] \geq \varepsilon^{k_n} \kappa_n$ for some $\kappa_n>0$.

\end{proof}
Now we can draw a similar conclusion for the case where $\mathcal{D}$ is non-convex.
\begin{lemma}\label{lem:1}
Let $\mathcal{D}$ be a non-convex connected domain. Let some $a\in\mathcal{A}$ have $\mathbf{x}_a(k_0) = \mathbf{x}_0$. Then, for any convex region $\mathcal{D}_n\subset \mathcal{D}$ of non-zero volume, there exists $\tau>0$ and $\kappa_n>0$ such that we can find a finite horizon $k_n\in\mathbb{N}$:
\begin{equation}
\Pr[\mathbf{x}_a(k_n+k_0)\in\mathcal{D}_n\,\mid\,\mathbf{x}_a(k_0)]\geq \varepsilon^{k_n} \kappa_n.
\end{equation}
\end{lemma}
\begin{proof}[Sketch of proof]
If $\mathcal{D}$ is connected, then for any two points $\mathbf{x}_0,\mathbf{x}_n\in\mathcal{D}$, we can construct a sequence of balls $\{\mathcal{X}_0,\mathcal{X}_1,...,\mathcal{X}_{k_n}\}$ of radius $R\geq \tau$ centred at $\mathbf{x}_0,\mathbf{x}_1,...,\mathbf{x}_{k_n}$ such that the intersections $\mathcal{X}_i\cap\mathcal{X}_{i+1}\neq\emptyset$ and are open sets, and $\mathbf{x}_{i}\in\mathcal{X}_{i-1}$. Then, we can pick $\tau$ to be small enough such that every ball $\mathcal{X}_i\subset \mathcal{D}$ does not intersect with the boundary of $\mathcal{D}$, and we can apply now Proposition \ref{prop:explore} recursively at every ball. If $\|\mathbf{x}_i-\mathbf{x}_{i-1}\|_2 < 2\tau$, then from Proposition \ref{prop:explore} we know $p(\mathbf{x}_i,k_{i-1}+2)>0$ since, for a given $\mathbf{x}_{i-1}$ and $k_{i-1}$, any point $\mathbf{x}_{i}$ has a non-zero probability density in at most 2 steps. Then, it holds that $p(\mathbf{x}_n,k_{n-1}+2)>0$ for some $k_{n-1}\in [k_n, 2 k_n]$, and for a target region $\mathcal{D}_n$: $\Pr[\mathbf{x}_a(k_n+k_0)\in\mathcal{D}_n\,\mid\,\mathbf{x}_a(k_0)] \geq \varepsilon^{k_n} \int_{\mathcal{D}_n}p_a(\mathbf{x}_n,k_n+k_0) d\mathbf{x}\geq \varepsilon^{k_n} \kappa_n$ for some $\kappa_n>0$.
\end{proof}
It follows from Lemma \ref{lem:1} that for a finite domain $\mathcal{D}$ every forager agent visits every region infinitely often. 

For a given initial combination of foraging agents, we have now guarantees that the entire domain is be explored and covered by beacons as $t\to\infty$.

\subsection{Foraging}
We leave for future work the formal guarantees regarding the expected weight field values $\omega_b^s(k)$ and guiding velocities $\mathbf{u}_b^s(k)$. Based on existing literature \cite{payton2001pheromone,djarhscc}, one could model the network of beacons as a discrete graph with a stochastic vertex weight reinforcement process based on the movement of the robots to proof convergence of both $\omega_b^s(k)$ and $\mathbf{u}_b^s(k)$. Such results would also allow us to study the limiting distribution of the robots across the space, as well as their trajectories. In this work we have experimentally verified that, over a diverse variety of set-ups, the weights $\omega_b^s(k)$ converge on average to a fixed point forming a gradient outwards from $\mathcal{S}$ and $\mathcal{T}$, therefore guiding the swarm to and from the goal regions.

\section{Experiments}\label{sec:exp}
For the experiments we present an application running on real Elisa-3 robot (GCTronic), and a more extensive statistical analysis of simulations running on \emph{Webots} \cite{michel2004cyberbotics}. We used the \emph{Webots} simulator for the implementation of the work in a robotic swarm, given its capabilities to simulate realistic robots in physical environments, and since it includes realistic Elisa-3 robot models, which  characteristics satisfy the restrictions in Assumption \ref{as:1} (using odometry for direction of motion measures).

The swarm agents are able to listen almost-continuously, store any incoming signals in a buffer, and empty the buffer every $\tau$ seconds. The parameters used for both the simulation and the real swarm are presented in Table \ref{table:first}.
\begin{table}[h]\centering
\begin{tabular}{ |c|c|c|c|c|c|c|c| }
  \hline
  \multicolumn{8}{|c|}{Parameters} \\
  \hline
  $\rho$ & $\lambda$&$r$& $\varepsilon$&$\tau\, (s)$&$\delta\,(m)$ & $v_0 (\frac{m}{s})$ & $\sigma^2$ \\
  \hline
  $0.01$ & $0.8$ & $1$ & $0.05$& $1$& $0.4$ & $0.25$ & $0.01$\\
  \hline
\end{tabular}
\vspace{2mm}
\caption{\label{table:first} Simulation parameters}
\end{table}
On both scenarios, the maximum range $\delta$ is virtually limited to restrict the communication range and simulate larger environments where not all robots are capable of reaching each-other. Additionally, the robots use a built-in collision avoidance mechanism, which is triggered when the infra-red sensors detect an obstacle closer than 2cm.

\subsection{Simulated Swarm}
We simulated first a general scenario where a swarm of size $N$ is deployed around a starting region and needs to look for a target in an arena of size $2.5m\times 3 m$, with and without obstacles. The robots act according to the rules in Algorithms \ref{al:1}, \ref{al:2}, which correspond to the dynamics covered in Section \ref{sec:swarmdesign}. In these first simulations we keep beacon agents static and without any switching.
\begin{algorithm}
\caption{Behaviour of Beacons}\label{al:1}
\begin{algorithmic}
\While{$s_b(k)=0$}
\State Broadcast $\omega_b^s(k),\mathbf{u}_b^s(k)$;
\State Listen for signals during $\tau$ seconds;
\State Compute $\omega_b^s(k+1),\mathbf{u}_b^s(k+1)$;
\State Move according to $\mathbf{v}_b(k+1)$;
   		\If{\emph{Obstacle}}
   			\State Move to avoid obstacle;
   			\EndIf
\State Check transitions in \eqref{eq:statetrans};
\EndWhile
\end{algorithmic}
\end{algorithm}
\begin{algorithm}
\caption{Behaviour of Foragers}\label{al:2}
\begin{algorithmic}
   \While{$s_f(k)\neq 0$}
\State Listen for signals during $\tau$ seconds;
\State Broadcast $\tilde{\mathbf{v}}_f(k,\sigma),\Delta_f^s(k)$;	
\State Compute $\mathbf{v}_f(k+1)$ from \eqref{eq:prv} and \eqref{eq:v2};
\State Move according to $\mathbf{v}_f(k+1)$;
\State \Comment{Asynchronous action}
   		\If{\emph{Obstacle}}
\State Move to avoid obstacle;
   			\EndIf
\State Check transitions in \eqref{eq:statetrans};
\EndWhile
\end{algorithmic}
\end{algorithm}

A first result on the swarm configuration under these parameters is shown in Figures \ref{fig:swarm1}, \ref{fig:swarm2}. 
\begin{figure}
     \centering
     \begin{subfigure}[b]{\linewidth}
         \centering
         \includegraphics[width=0.8\linewidth]{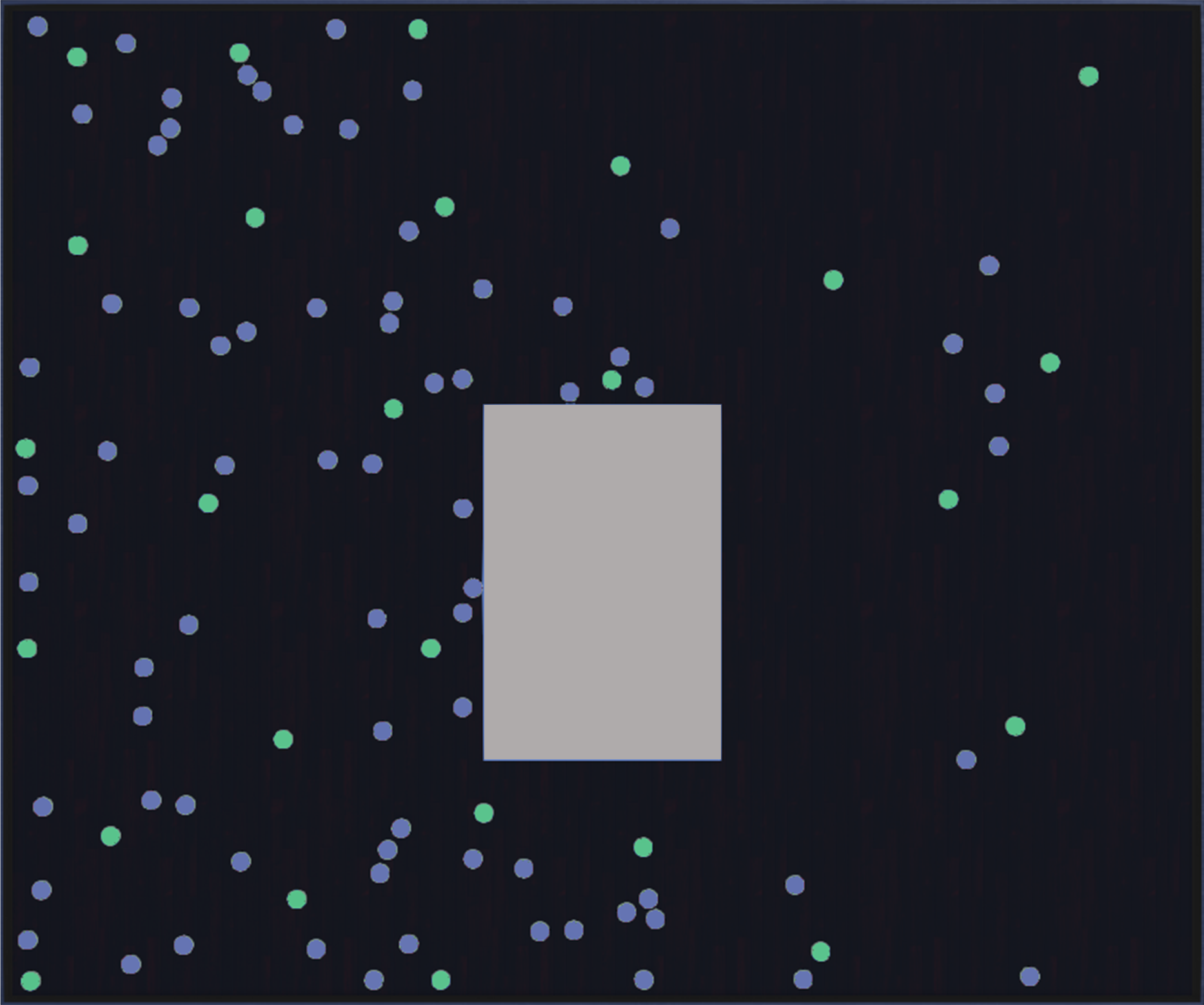}
         \caption{Real Swarm}
         \label{fig:real1}
     \end{subfigure}
     \hfill
     \begin{subfigure}[b]{\linewidth}
         \centering
         \includegraphics[width=0.8\linewidth]{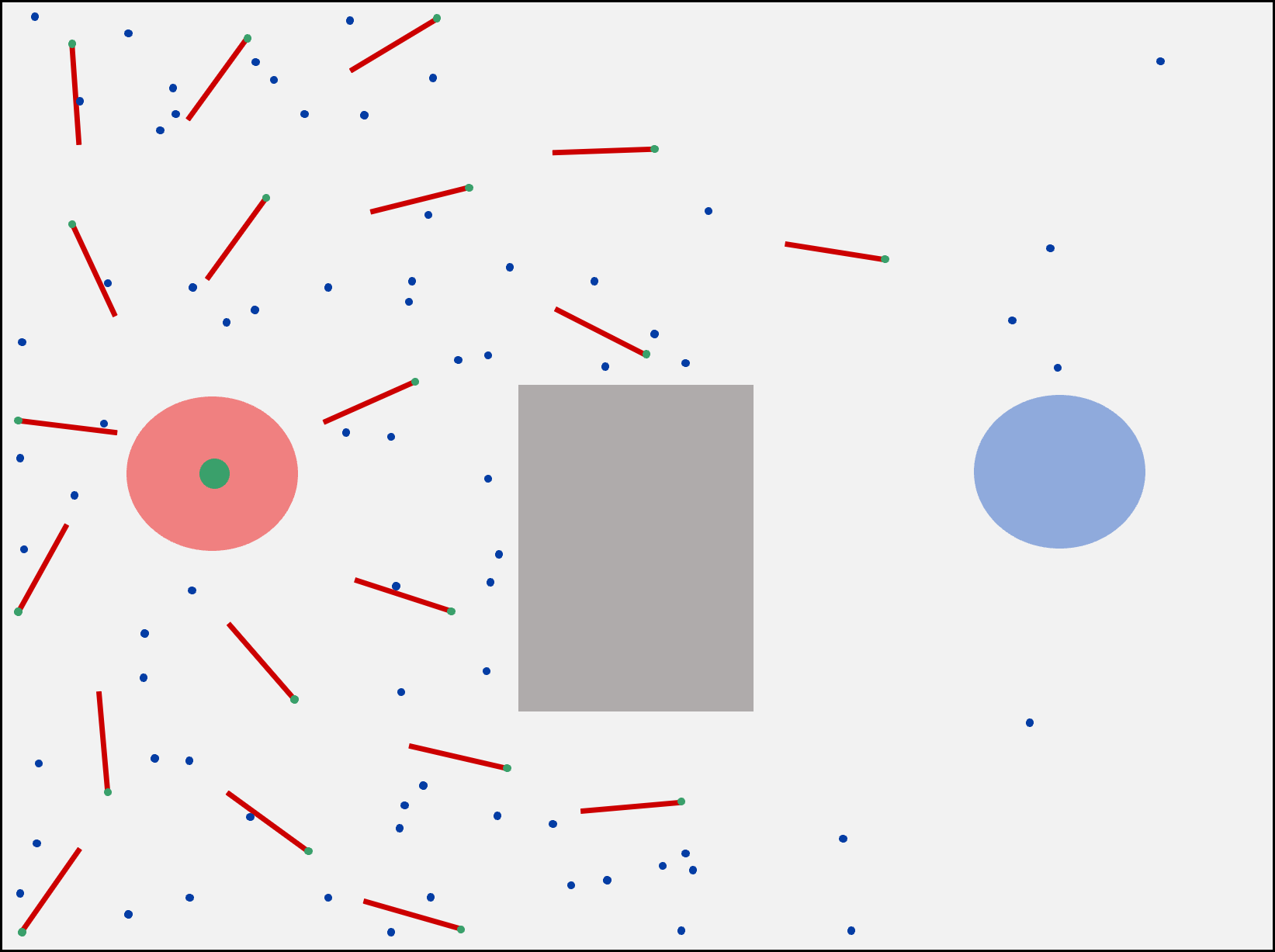}
         \caption{Swarm State}
         \label{fig:state1}
     \end{subfigure}
        \caption{Foraging Swarm with obstacle at $t=50s$}
        \label{fig:swarm1}
\end{figure}

\begin{figure}
     \centering
     \begin{subfigure}[b]{\linewidth}
         \centering
         \includegraphics[width=0.8\linewidth]{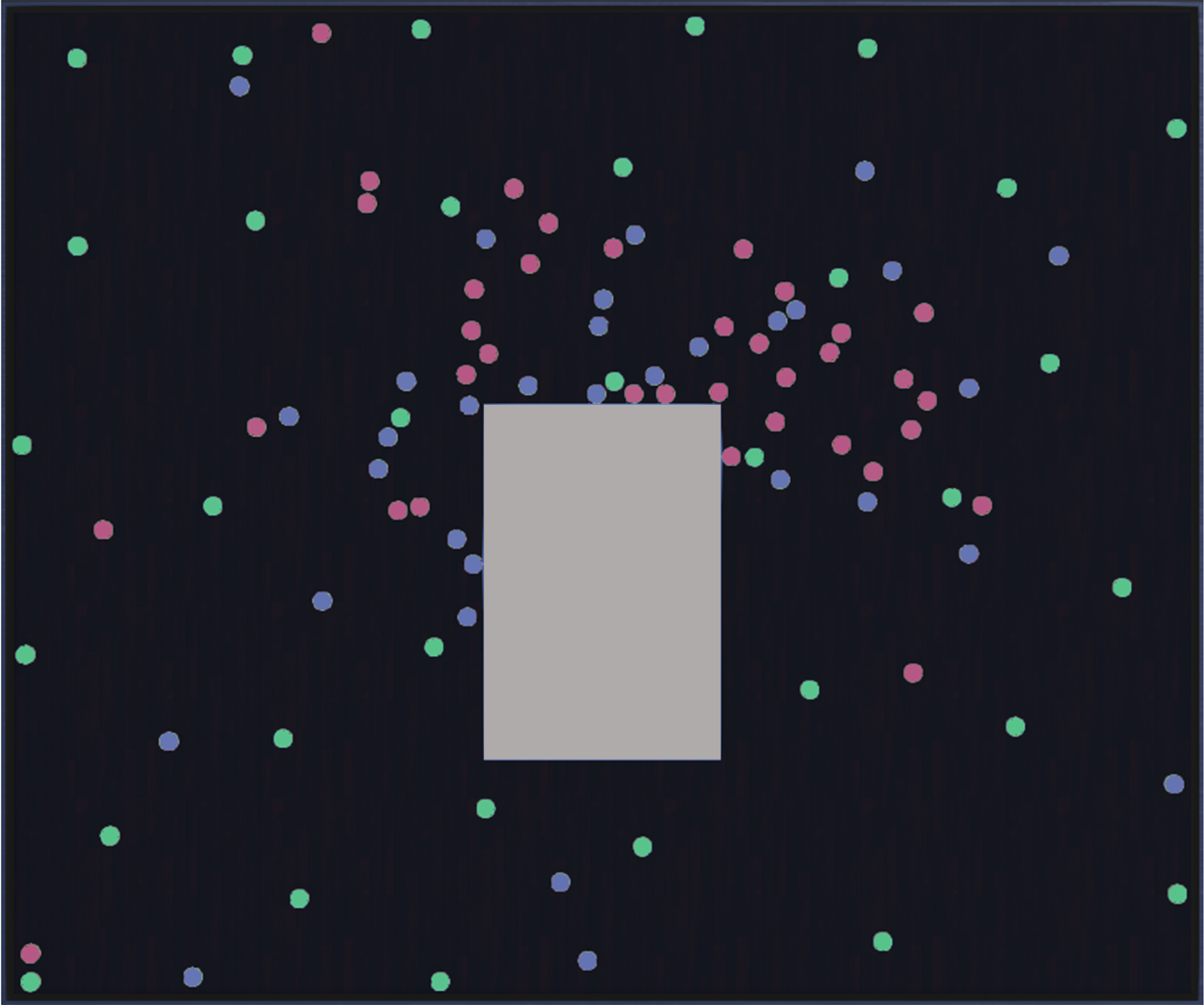}
         \caption{Real Swarm}
         \label{fig:real2}
     \end{subfigure}
     \hfill
     \begin{subfigure}[b]{\linewidth}
         \centering
         \includegraphics[width=0.8\linewidth]{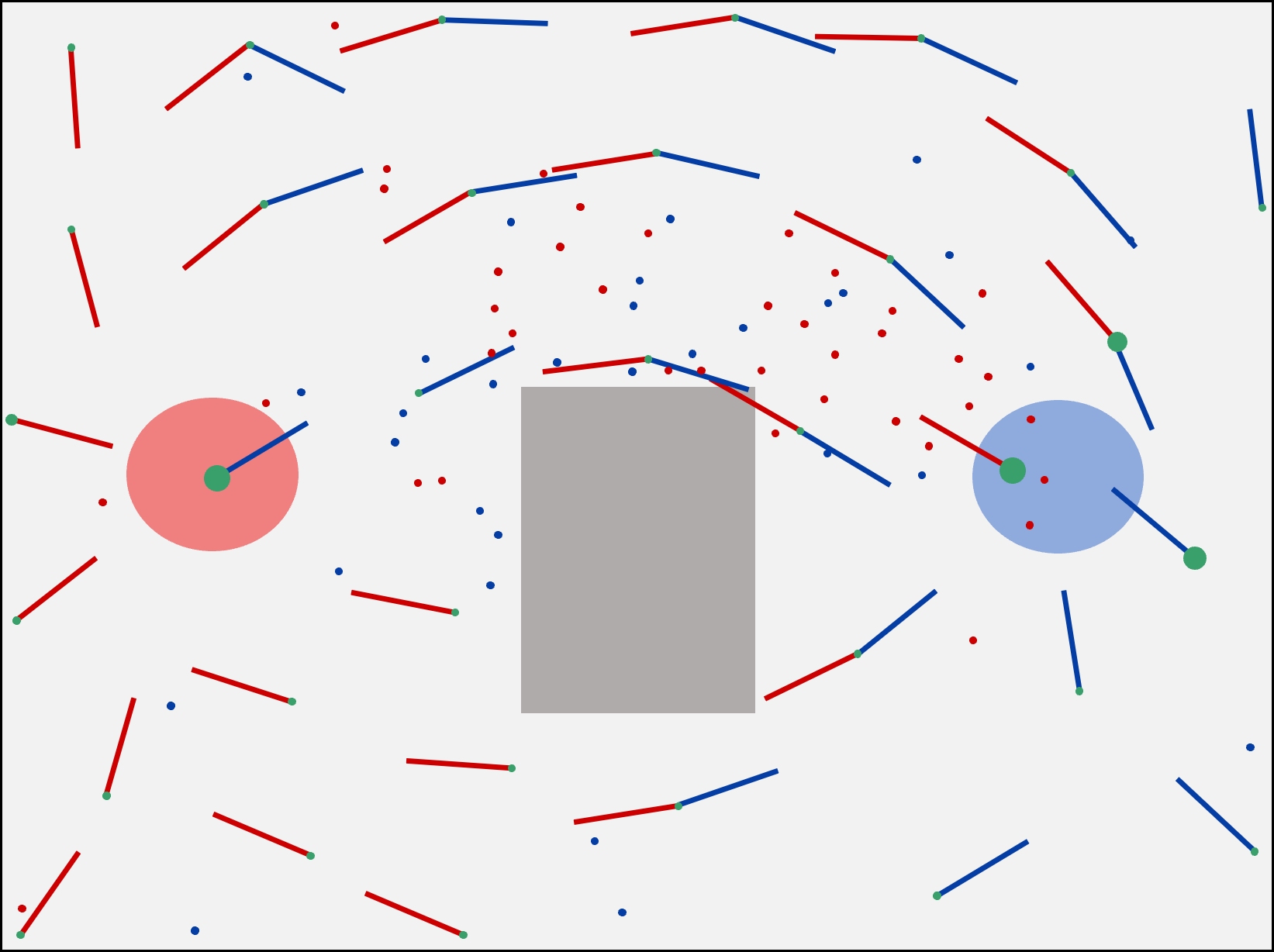}
         \caption{Swarm State}
         \label{fig:state2}
     \end{subfigure}
        \caption{Foraging Swarm with obstacle at $t=300s$}
        \label{fig:swarm2}
\end{figure}

The robots are released in batches, to avoid overcrowding in the starting region. Figures \ref{fig:real1} and \ref{fig:real2} are the real swarm on the Webots simulator, with blue pucks being beacons, green and red pucks being foragers in different states. Figures \ref{fig:state1} \ref{fig:state2} are a simplified visualization of the exact same frame, including information on $\omega(k),\,\mathbf{u}(k)$. The green circles are the nest regions, red circles are the target regions. The size of black dots represent relative amount of weights stored at the beacons, and the plotted vectors are the guiding velocities $\mathbf{u}^s(k)$ at each beacon. 

One can see how the swarm is able to navigate to the target set, and does so around the shortest path. Note however how the random motion still affects some members of the swarm, and how the agents' accumulation seems to be restricted by how well the beacons represent the desired velocity field. This can be clearly seen in Figure \ref{fig:state2}. Most swarm members are indeed moving back and forth from $\mathcal{S}$ to $\mathcal{T}$, but do so spreading between the rows of beacons closer to the minimum length path.

\subsubsection{Performance Metrics}
To fully evaluate the performance of the method we now include results regarding convergence and efficacy metrics. We are interested in measuring two main properties in such swarm:
\begin{enumerate}
\item Foraging performance. 
\item Accumulation of agents around foraging trajectories.
\end{enumerate}
To measure the foraging performance, we use the navigation delay of the swarm. Navigation delay is used as a performance measure of foraging problems \cite{ducatelle2011communication}, and it is defined as the average time it takes an agent to travel back and forth between the target regions. For a finite horizon $t\in[t_0,T]$,
\begin{equation}
d(t_0,T) = \frac{1}{\lvert\mathcal{A} \rvert}\sum_{i\in\mathcal{A}}\frac{T-t_0}{\# \text{trips}_i}.
\end{equation} 

Accumulation of agents around trajectories (or path creation) is an example of emergent self-organized behaviour. To measure this, entropy has been used to quantize forms of clustering in robots, and to investigate if stable self-organization arises. We use the hierarchic social entropy as defined by \cite{balch2000hierarchic} using single linkage clustering. Two robots are in the same cluster if the relative position between them is smaller than $h$. Then, with $\mathcal{C}(t,h)$ being the set of clusters at time $t$ with minimum distance $h$, $c_i\in\mathcal{C}(t,h)$ and $\mathcal{A}_{c_i}$ being the subset of agents in cluster $c_i$, entropy for a group of robots is
\begin{equation}
H(\mathcal{A}) = - \sum_{c_i\in\mathcal{C}(t,h)}\frac{\lvert\mathcal{A}_{c_i}\rvert}{N} \log_2(\frac{\lvert\mathcal{A}_{c_i}\rvert}{N}).
\end{equation}
Hierarchic social entropy is then defined as integrating $H(R,h)$ over all values of $h$, $S(R) = \int_{\delta_0}^\infty H(R,h)dh.$ To allow for zero entropy, we take the lower limit of the integral to be the diameter of the robots ($\delta_0=4cm$).
\subsubsection{Results}
For our experiments, we consider swarm size between $N\in [49,157]$. We evaluate the system in an arena of $7.5m^2$, and study 2 cases (with and without central obstacle). The results for the navigation delay are presented in Figures \ref{fig:metrics} and \ref{fig:metricsnoobs}. All simulations are run for horizons $T=400s$, since that was observed to be enough for the swarm to stabilize.
\begin{figure}
     \centering
     \begin{subfigure}[b]{\linewidth}
         \centering
         \includegraphics[width=0.875\linewidth]{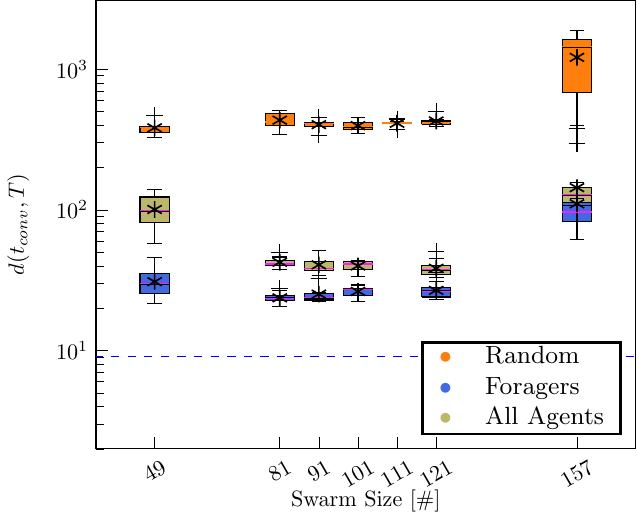}
         \caption{Swarm with symmetrical obstacle.}
         \label{fig:metrics}
         \vspace{0.5cm}
     \end{subfigure}
     \hfill
     \begin{subfigure}[b]{\linewidth}
         \centering
         \includegraphics[width=0.875\linewidth]{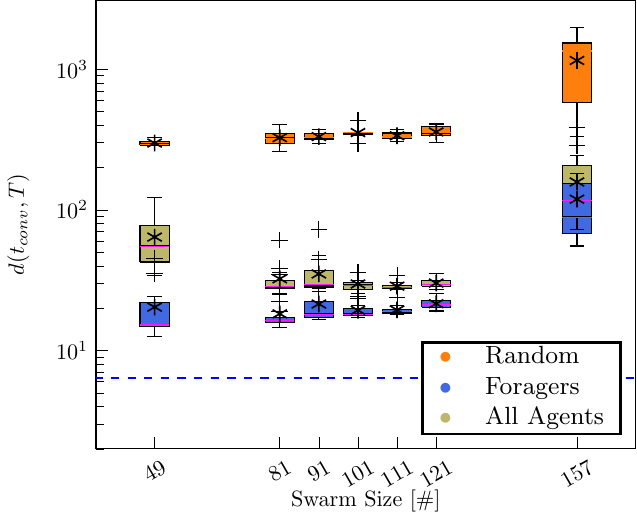}
         \caption{Swarm without obstacle.}
         \label{fig:metricsnoobs}
     \end{subfigure}
        \caption{Navigation times for $[t_{conv},T]$.}
\end{figure}
Each data point includes results from 12 independent test runs. The navigation delay is measured for three different scenarios: 
\begin{enumerate}
\item Orange: Agents moving at random for $(t_{\text{conv}},T_{\max})$.
\item Yellow: All agents $\mathcal{A}$ when following self-guided foraging for $(t_{\text{conv}},T)$.
\item Blue: Only forager agents $\mathcal{A}^s$ for the horizon $(t_{\text{conv}},T)$.
\end{enumerate}
To get a conservative value for $t_{\text{conv}}$, we define it as the first time step when an agent completes the first full foraging trip back and forth. We also include an absolute lower bound, corresponding to the absolute minimum possible travel time in the considered scenario. 

We can now extract several conclusions with respect to the swarm sizes. For a size of $N=49$ or smaller too many agents are needed as beacons, hence the performance (specially when considering the full swarm) is significantly worse than for bigger swarms. We can see that, for growing swarm sizes, the performance increases and the variance in the results reduces, until a certain point (around $N=120$) where this tendency reverses. This is due to the \emph{overcrowding} effect. The robots occupy space, and need to perform obstacle avoidance manoeuvres to not collide with each-other. The robots are around $4$cm in diameter, and start collision avoidance manoeuvres when they are close ($\approx 2cm$) to an obstacle or another robot. Since the experiments are run on a small arena, the swarm reaches a point where there can be too many robots in the same space. Another point worth noting is that the lower bound (blue dashed line) in Figures \ref{fig:metrics} and \ref{fig:metricsnoobs} is extremely conservative, and represents the case where robots know and follow the optimal path to the target.

\begin{figure}
     \centering     
     \begin{subfigure}[b]{\linewidth}
         \centering
         \includegraphics[width=0.9\linewidth]{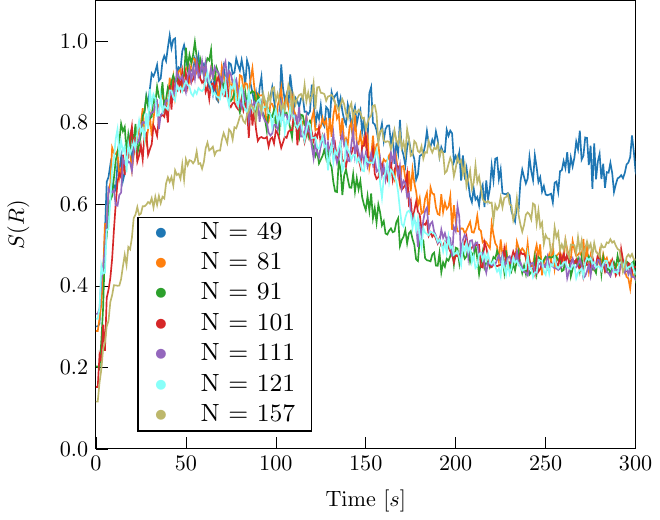}
         \caption{Swarm with symmetrical obstacle.}
         \label{fig:entropytrue}
         \vspace{0.5cm}
     \end{subfigure}
          \begin{subfigure}[b]{\linewidth}
         \centering
         \includegraphics[width=0.9\linewidth]{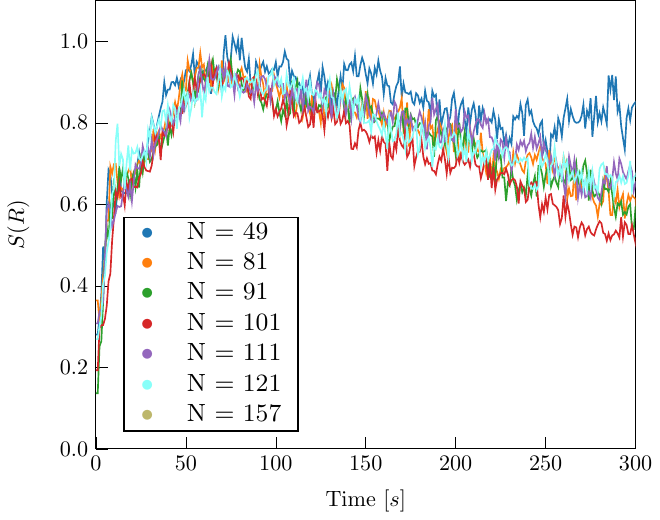}
         \caption{Swarm without obstacle.}
         \label{fig:entropyfalse}
     \end{subfigure}
        \caption{Social entropy with and without obstacles.}
\end{figure}
The entropy results are presented in Figure \ref{fig:entropyfalse} and \ref{fig:entropytrue}. In this case, we only compute the entropy of the forager agents, and we normalise against a practical "maximum" entropy computed by randomizing agents over the entire space. At $t=0$ all agents start at the nest, hence the low entropy values, and from there the entropy increases while the swarm explores and tries to find the target. After the exploration phase, the entropy begins to settle to lower values as the robots accumulate over fewer trajectories. Entropy is higher for the symmetrical obstacle due to first, the split of agents among the two possible paths, and second, the minimum length path being longer.

At last, regarding the robustness in the system, if a beacon fails or is removed, it will be eventually replaced by a new exploring agent. The simulations are performed for noisy signals (weights and velocities), and the beacon processes a maximum number of $5$ signals to account for bandwidth restrictions or colliding communication packages.
\subsection{Robotic Swarm}
We implemented the work on real Elisa3-robots, in order to qualitatively confirm results of the work. Since the Elisa-3 robot is also used for the Webots experiments, the robot characteristics and behaviour are the same for both the real swarm and the simulated swarm. Note that the Elisa3-robots matches the restrictions of Assumption \ref{as:1}, except the capacity to measure its angular orientation. To overcome this shortcoming, a global tracking system is employed which equips the robots with a virtual compass. Since our tracking system is only able to measure the orientations of the robots when all robots are not moving, the robots run synchronously.

In the test set-up, a swarm of 35 robots is deployed around a starting region and needs to look for a target in an arena of size $2.4 m \times 1.15 m$ without obstacles (See Electronic Suplementary Material 1). Figures \ref{fig:real robots1}, \ref{fig:real robots2}, \ref{fig:real robots3} and \ref{fig:real robots4} show snapshots of this simulation. The red and blue pillars are the centers of the nest region and target region, respectively. Green coloured robots are beacons, red coloured are searching for the nest region, blue coloured are searching for the target region. The black arrows indicate the heading direction, the blue and red arrows the guiding velocities $\mathbf{\nu}(k)$ at each beacon. 

The resulting behaviour aligns with the predicted design dynamics.
The swarm creates a covering beacon network. The first robots reaching the target region attracts other robots to this region and after 30 steps all non-beacon robots travel back and forth between the target regions, clustered around the shortest path. We point out that during all tests there were robot failures. This did not affected the swarm's behaviour noticeably, showing its robustness. We leave for future work the realisation of extensive tests with more powerful robots to confirm all results in other scenarios of swarm deployment. 
\begin{figure}
     \centering
     \begin{subfigure}[b]{\linewidth}
         \centering
         \includegraphics[width=0.8\linewidth]{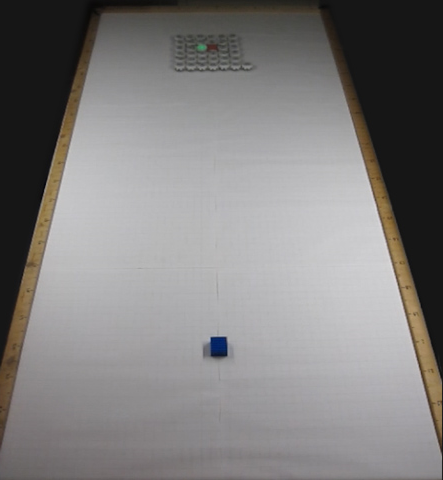}
         \caption{Initial Configuration}
         \label{fig:real robots1}
     \end{subfigure}
     \hfill
     \begin{subfigure}[b]{\linewidth}
         \centering
         \includegraphics[width=0.8\linewidth]{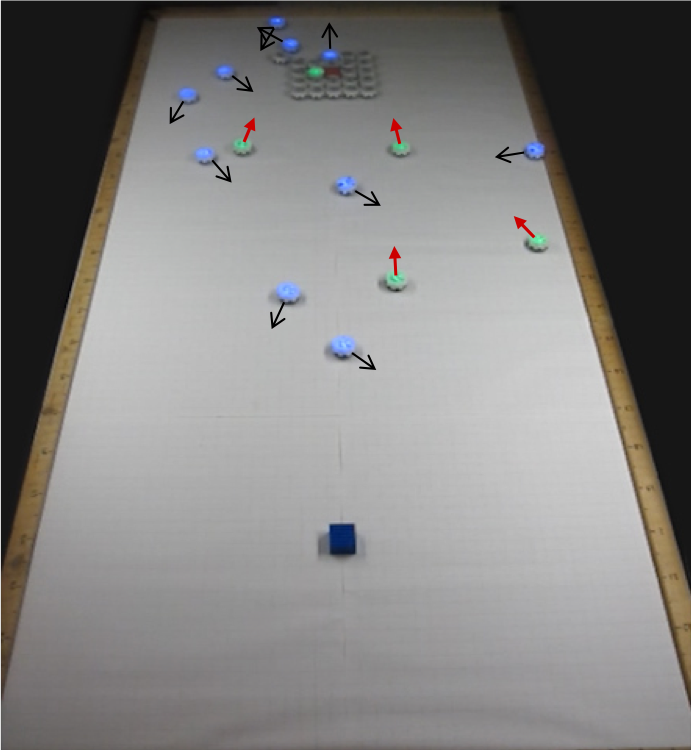}
         \caption{Exploration}
         \label{fig:real robots2}
     \end{subfigure}
     \caption{Swarm in exploration mode.}
\end{figure}

\begin{figure}
    \centering
     \begin{subfigure}[b]{\linewidth}
         \centering
         \includegraphics[width=0.8\linewidth]{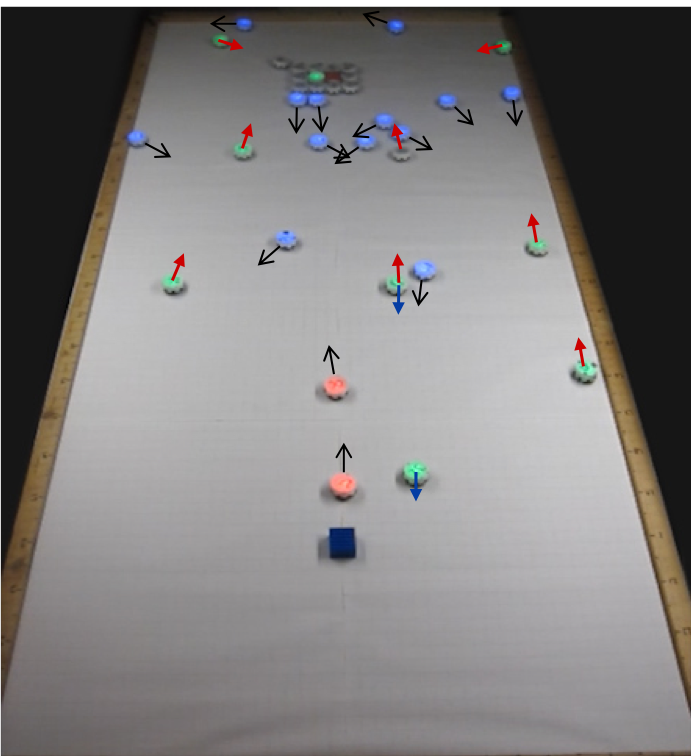}
         \caption{Target Found}
         \label{fig:real robots3}
     \end{subfigure}
     \hfill
     \begin{subfigure}[b]{\linewidth}
         \centering
         \includegraphics[width=0.8\linewidth]{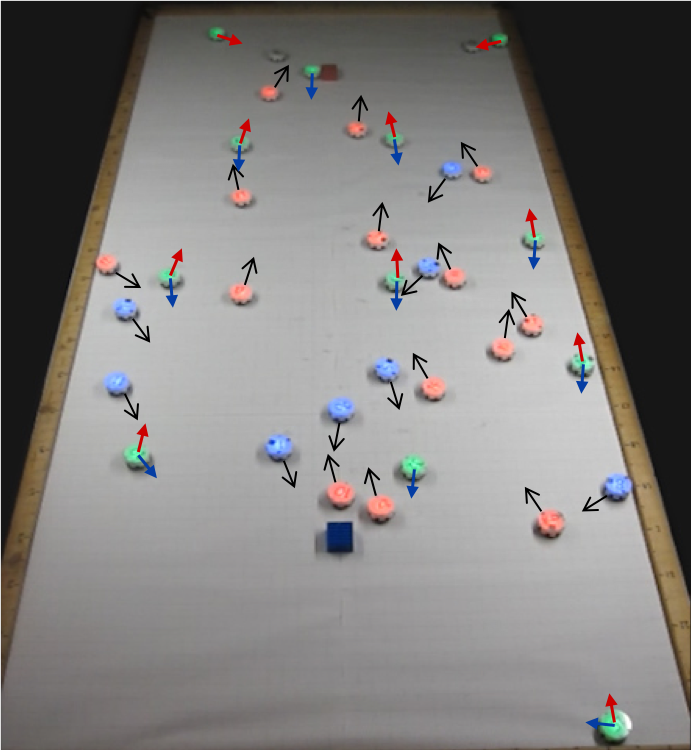}
         \caption{Exploitation}
         \label{fig:real robots4}
     \end{subfigure}
     \caption{Swarm exploiting the food source.}
\end{figure}
\section{Extensions and future directions}
The presented system throughout this work has shown to be effective at generating emerging behaviour that solves an exploration-exploitation problem, while requiring very minimal characteristics on the individual robots. The ultimate goal is to design swarms methodologies that allow us to implement such methods in real robots as efficiently as possible, without requiring complex single systems. In this line of thought, we can already propose some extensions to such minimalistic swarm methods. One clear downside of the foraging system presented in this work is the requirement for relatively large amounts of beacons. This increases the overall inefficiency and requirements of the system, requiring robots to behave as beacons even in regions where they are not needed.

An extension to our method to increase overall efficiency could be to implement controllers in the beacons to allow them to \emph{turn back to foragers} when their weights have not been updated for a long time and to \emph{move to more transited areas} of the environment. This would on the one hand ensure only the necessary amount of beacons are employed as the system evolves, and on the other hand increase the granularity (or the definition) of the paths that are being used more often, possibly enabling more optimal configurations.

We implemented a prototype of such method, simply allowing beacons to turn to foragers when their weights are lower than a specific threshold for too long, and adding a P controller to the velocity of the beacons (set to $0$ in previous experiments). This controller allows the beacons to slowly move towards the mid direction between their two guiding velocity vectors $\mathbf{u}_b^s(k)$. The logic between this controllers is straight-forward: If the most optimal path is a straight line, the most optimal configuration for the beacons would be to sit along the line, with the guiding vectors on $180^o$.
\begin{figure}
    \centering
     \begin{subfigure}[b]{\linewidth}
         \centering
         \includegraphics[width=0.8\linewidth]{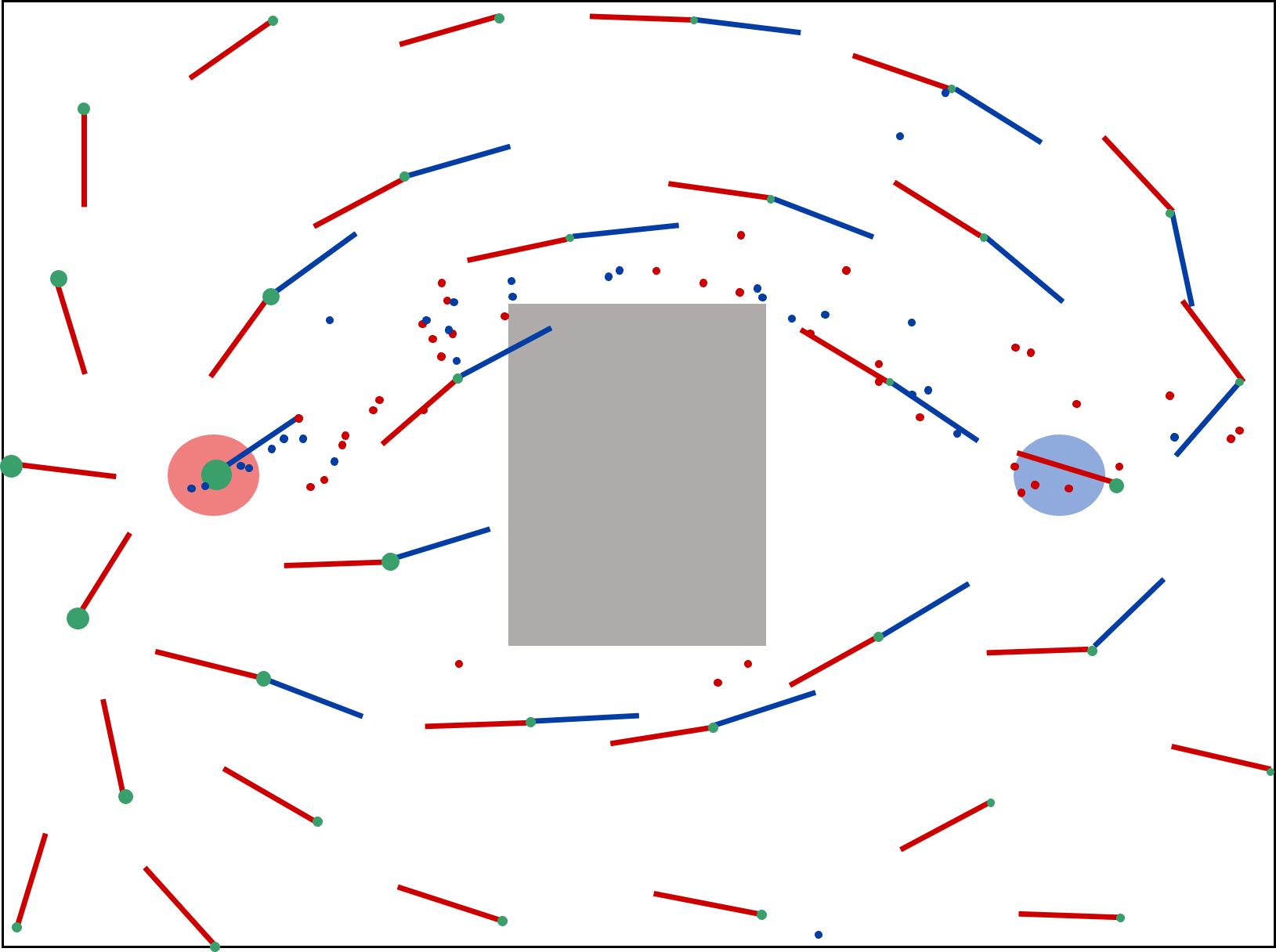}
         \caption{Adapting swarm after $100s$.}
         \label{fig:moving1}
     \end{subfigure}
     \hfill
     \begin{subfigure}[b]{\linewidth}
         \centering
         \includegraphics[width=0.8\linewidth]{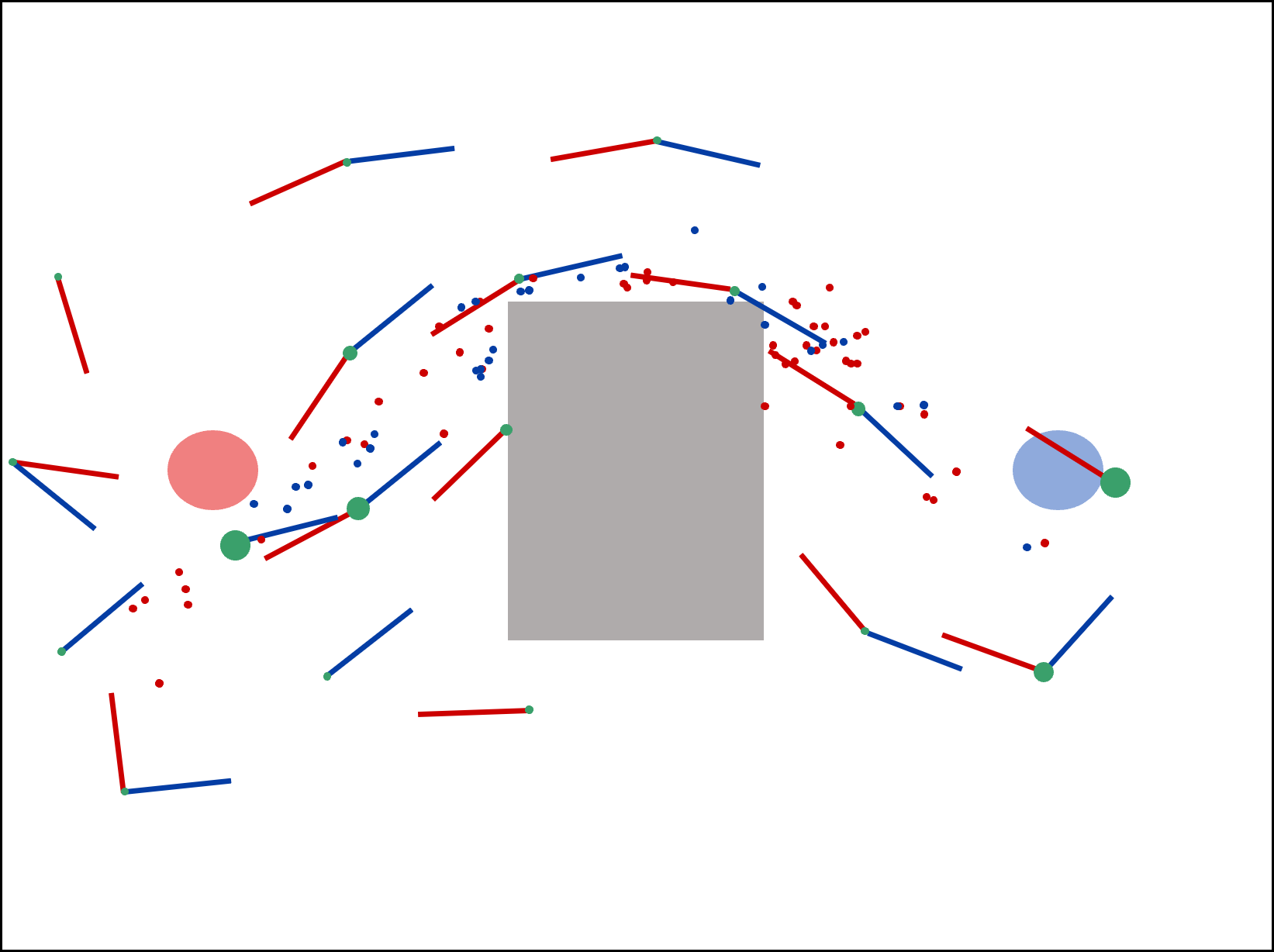}
         \caption{Adapting swarm after $200s$.}
         \label{fig:moving2}
     \end{subfigure}
     \caption{Moving beacons extension.}
\end{figure}
The extension was implemented on a simplified ``particle" simulator with $N=101$ agents, where the robots are point-masses and are not affected by collision avoidance mechanisms, since this allows us to test extensions quicker before moving to real implementations. From Figures \ref{fig:moving1} and \ref{fig:moving2} one can see how this proposal seems promising to increase the efficiency of the system. Beacons re-configure themselves, and get closer to the optimal path allowing more accurate trajectories by the agents. Additionally, the environment is covered by beacons during the exploration phase, but eventually large unvisited areas of the domain are covered less beacons (until some are completely empty), reflecting the fact that those areas are not of interest any more. Such extensions would also have an impact on the performance of the system on dynamic environments, by allowing the swarm to re-configure depending on changes around them, and using more infrastructure (beacons) when the surroundings are changing faster.
\section{Discussion}
We have presented a foraging swarm design where robots are able to guide each-other to successfully forage without the need of position measurements, infrastructure or global knowledge, and using minimal amounts of knowledge about each-other or the environment. The system has been implemented on a swarm of Elisa-3 robots, and an extensive experimental study has been performed using the \emph{Webots} simulator. We have shown how a middle sized swarm ($N\approx 100$) is able to find an unknown target and exploit trajectories close to the minimum length path. The system does require agents to know their orientation, and we have seen how it can be affected by overcrowding effects when agents need to avoid colliding with each-other. Additionally, we have observed how the optimality of the trajectories is affected by the resulting distribution of the beacons, which gives room for future work regarding the possibility of having beacons re-configure to more optimal positions. This would allow for beacon movement and re-configuration to explore the possibilities of using dynamic infrastructure in robotic swarms. We expect in the future to add ultra-wide band communication modules to the Elisa-3 robots with magnetometers, that would allow us to run the system on a much larger swarm and on more complex environments, and to apply it to other navigation-based problems like surveillance, target interception or flocking. We leave for future work as well the formal analysis of the resulting beacon graphs, and the evolution of the variables $\omega(k)$ and $\mathbf{u}(k)$.

\bibliographystyle{IEEEtran}
\bibliography{IEEEabrv,biblo}

\end{document}